\documentclass[letterpaper]{article}
\usepackage[preprint]{aaai2027}
\usepackage[hyphens]{url}
\usepackage{graphicx}
\usepackage{natbib}
\usepackage{caption}
\usepackage{amsmath}
\usepackage{amssymb}
\usepackage{booktabs}
\usepackage{array}
\usepackage{xcolor}
\usepackage{soul}
\usepackage{environ}
\usepackage{tikz}
\usetikzlibrary{calc,decorations.pathmorphing,positioning,backgrounds,arrows.meta}

\definecolor{takeawayink}{HTML}{E9DDC0}
\definecolor{catscale}{HTML}{DDE5EE}
\definecolor{cattables}{HTML}{E0E8DC}
\definecolor{catabst}{HTML}{F0E0D5}
\definecolor{cataggr}{HTML}{E5DFEC}
\definecolor{figcream}{HTML}{F1EDE3}
\definecolor{figink}{HTML}{3D3929}
\definecolor{figgray}{HTML}{8A8574}
\definecolor{bulblue}{HTML}{8FA8C9}
\definecolor{bulsage}{HTML}{9DB891}
\definecolor{bulclay}{HTML}{C98F70}
\definecolor{bullav}{HTML}{A99BC7}
\pgfmathsetseed{27}
\soulregister\ref7
\soulregister\eqref7
\soulregister\MET0
\soulregister\UNMET0
\soulregister\CA0
\soulregister\CANNOTASSESS0

\makeatletter
\tikzset{every takeaway stroke/.style={fill=takeawayink, blend mode=multiply}}
\newcommand{\takeaway@DoHighlight}{%
  \fill [ decoration = {random steps, amplitude=1pt, segment length=12pt}
        , outer sep = -15pt, inner sep = 0pt, decorate
        , every takeaway stroke ]
        ($(begin takeaway)+(0,7.2pt)$) rectangle ($(end takeaway)+(0,-2.8pt)$) ;
}
\newcommand{\takeaway@Begin}{\coordinate (begin takeaway) at (0,0) ;}
\newcommand{\takeaway@End}{\coordinate (end takeaway) at (0,0) ;}
\newdimen\takeaway@previous
\newdimen\takeaway@current
\newdimen\takeaway@previousx
\newdimen\takeaway@currentx
\newif\iftakeaway@inside
\newcommand{\takeaway@markend}{\tikz[overlay, remember picture] \takeaway@End ;}
\newcommand{\takeaway@checkbreak}{%
  \begin{tikzpicture}[overlay, remember picture]
    \path let \p0 = (begin takeaway), \p1 = (0,0) in \pgfextra
      \global\takeaway@previous=\y0
      \global\takeaway@current =\y1
      \global\takeaway@previousx=\x0
      \global\takeaway@currentx=\x1
    \endpgfextra (0,0) ;
    \ifdim\takeaway@current=\takeaway@previous\else
      \ifdim\dimexpr\takeaway@previous-\takeaway@current\relax<-50pt
        \ifdim\takeaway@previousx>\takeaway@currentx
          \takeaway@Begin
        \else
          \takeaway@DoHighlight
          \takeaway@Begin
        \fi
      \else
        \takeaway@DoHighlight
        \takeaway@Begin
      \fi
    \fi
  \end{tikzpicture}%
}
\DeclareRobustCommand*\takeawayhighlight{%
  \SOUL@setup
  \takeaway@insidetrue
  \def\SOUL@preamble{%
    \begin{tikzpicture}[overlay, remember picture]
      \takeaway@Begin
      \takeaway@End
    \end{tikzpicture}%
  }%
  \def\SOUL@postamble{%
    \begin{tikzpicture}[overlay, remember picture]
      \takeaway@End
      \takeaway@DoHighlight
    \end{tikzpicture}%
  }%
  \def\SOUL@everyhyphen{%
    \discretionary{%
      \SOUL@setkern\SOUL@hyphkern
      \SOUL@sethyphenchar
      \tikz[overlay, remember picture] \takeaway@End ;%
    }{%
    }{%
      \SOUL@setkern\SOUL@charkern
    }%
  }%
  \def\SOUL@everyexhyphen##1{%
    \SOUL@setkern\SOUL@hyphkern
    \hbox{##1}%
    \discretionary{%
      \tikz[overlay, remember picture] \takeaway@End ;%
    }{%
    }{%
      \SOUL@setkern\SOUL@charkern
    }%
  }%
  \def\SOUL@everysyllable{%
    \takeaway@checkbreak
    \the\SOUL@syllable
    \takeaway@markend
  }%
  \SOUL@
}
\makeatother

\NewEnviron{takeaway}{%
  \unskip\space
  {\toks0={\mbox{\textbf{Takeaway:}} }%
   \toks2=\expandafter{\BODY}%
   \edef\takeawaytmp{\noexpand\takeawayhighlight{\the\toks0 \the\toks2}}%
   \takeawaytmp}%
}

\makeatletter
\g@addto@macro\normalsize{%
  \setlength{\abovedisplayskip}{4pt plus 2pt minus 2pt}%
  \setlength{\belowdisplayskip}{4pt plus 2pt minus 2pt}%
  \setlength{\abovedisplayshortskip}{0pt plus 2pt}%
  \setlength{\belowdisplayshortskip}{2pt plus 1pt minus 1pt}%
}
\makeatother

\newtheorem{fact}{Fact}
\newtheorem{proposition}{Proposition}

\makeatletter
\newcommand{\verdicttoken}[1]{\iftakeaway@inside\takeaway@checkbreak\fi\texttt{#1}\iftakeaway@inside\takeaway@markend\fi}
\makeatother
\newcommand{\MET}{\verdicttoken{MET}}
\newcommand{\UNMET}{\verdicttoken{UNMET}}
\newcommand{\CA}{\verdicttoken{CA}}
\newcommand{\CANNOTASSESS}{\verdicttoken{CANNOT\_ASSESS}}
\newcommand{\cat}[1]{\textsc{#1}\enspace}

\title{Agreement Measurement for Rubric-based LLM Judges:\\ What to Report and Why}
\author{Delip Rao\thanks{Corresponding author.}, Chris Callison-Burch}
\affiliations{University of Pennsylvania\\
delip@seas.upenn.edu, ccb@seas.upenn.edu}

\begin{document}
\maketitle

\begin{abstract}
  Whether a rubric-based LLM judge can replace human annotation is decided by its measured agreement with human labels. Yet the same verdicts can support wildly varying agreement numbers, depending on seemingly minor choices: the judgment scale, the retained cases, the handling of abstentions and invalid outputs, and the pooling of verdicts across items and rubric criteria. The statistics that settle these choices are established, but in psychometrics, econometrics, and corpus annotation rather than in the evaluation practice that needs them. We treat the choices as a measurement protocol that fixes what the reported number estimates before any metric is computed, assemble the relevant results into a single source-attributed analysis, and apply it to three published LLM-judge evaluations. For non-degenerate binary verdicts, Pearson's $r$, Spearman's $\rho$, Kendall's $\tau_b$, the phi coefficient, and the Matthews correlation coefficient are exactly the same statistic, so reporting several repeats one number under different names. Cohen's $\kappa$ differs from them only through a marginal-mismatch factor in $(0,1]$ and shares their asymptotic variance when judge and human assign the positive verdict equally often. Under exclusion, accuracy over all cases is pinned down only to a worst-case interval as wide as the uncovered fraction. On a rubric benchmark carrying per-criterion human labels, protocol choice alone moves reported accuracy from $0.551$ to $0.899$ and carries $\kappa$ across zero, without altering a single verdict. We distill the analysis into a reporting checklist that makes agreement claims reconstructible and comparable.
\end{abstract}

\section{Introduction}
\label{sec:intro}

\begin{figure}[t!]
\centering
\newcommand{\figct}[1]{{\fontsize{6.5}{8.5}\selectfont\bfseries #1}}
\newcommand{\figci}[2]{\textcolor{#1}{\raisebox{0.2ex}{\scalebox{0.6}{$\bullet$}}}\,#2}
\newcommand{\figrl}[1]{{\fontsize{5}{7}\selectfont\color{figgray}\scshape #1}}
\newcommand{\figrv}[1]{{\fontsize{6}{7}\selectfont #1}}
\scalebox{1.0}{%
\begin{tikzpicture}[
  card/.style={rounded corners=3pt, inner sep=4pt, align=left, text=figink,
               font=\fontsize{6}{7.5}\selectfont},
  cardnote/.style={font=\fontsize{5.5}{6.5}\selectfont\itshape, text=figgray,
                   inner sep=1.5pt},
  arr/.style={-{Stealth[length=3.5pt]}, figgray, line width=0.6pt},
  arrlab/.style={font=\fontsize{5.5}{6.5}\selectfont\itshape, text=figgray, inner sep=1.5pt},
]
\node[card, fill=figcream, inner sep=5pt, align=center] (c) {%
  {\fontsize{7.5}{9}\selectfont\bfseries Reconstructible}\\[1.5pt]
  {\fontsize{7.5}{9}\selectfont\bfseries agreement result}\\[2pt]
  {\color{figgray}\rule{2.4cm}{0.3pt}}\\[2.5pt]
  \begin{tabular}{@{}l@{\hspace{4pt}}l@{}}
  \figrl{protocol}  & \figrv{$\mathcal{P}=(s,\mathcal{U},h,a)$} \\
  \figrl{statistic} & \figrv{$T$} \\
  \figrl{reported}  & \figrv{estimate $+$ interval} \\
  \end{tabular}};
\node[card, fill=catscale, above=0.35cm of c] (top) {%
  \figct{Scale \& target}\\[1.5pt]
  \figci{bulblue}{judgment scale}\\
  \figci{bulblue}{matched metric}\\
  \figci{bulblue}{both \MET{} rates}};
\node[card, fill=cataggr, below=0.35cm of c] (bot) {%
  \figct{Aggregation \& inference}\\[1.5pt]
  \begin{tabular}{@{}l@{\hspace{6pt}}l@{}}
  \figci{bullav}{micro / macro / item} & \figci{bullav}{weights, denominators}\\
  \figci{bullav}{resampling unit}      & \figci{bullav}{ensemble tables}\\
  \end{tabular}};
\node[card, fill=cattables, left=0.32cm of c] (lft) {%
  \figct{Evidence}\\[1.5pt]
  \figci{bulsage}{$2{\times}2$ table $+\,N$}\\
  \figci{bulsage}{$FP$ versus $FN$}\\
  \figci{bulsage}{degenerate $=$ NA}};
\node[card, fill=catabst, right=0.32cm of c] (rgt) {%
  \figct{Scope}\\[1.5pt]
  \figci{bulclay}{\CA{} / tie / invalid}\\
  \figci{bulclay}{coverage rates}\\
  \figci{bulclay}{covered / full set}};
\draw[arr] (top.south) -- (c.north) node[arrlab, midway, right=1pt] {defines the target};
\draw[arr] (bot.north) -- (c.south) node[arrlab, midway, right=1pt] {defines the unit};
\draw[arr] (lft.east) -- (c.west);
\draw[arr] (rgt.west) -- (c.east);
\node[cardnote, below=1.5pt of lft.south] {determines every metric};
\node[cardnote, below=1.5pt of rgt.south] {conditions the population};
\end{tikzpicture}}%
\caption{An agreement statistic becomes reconstructible only when it is surrounded by its target, its evidence, its scope, and its unit of inference. Evaluations leave these commitments implicit (Sec.~\ref{sec:related}); the omissions move reported numbers (Secs.~\ref{subsec:case-selfpref}, \ref{subsec:case-cascade}, \ref{subsec:case-rubric}); Secs.~\ref{sec:protocol}--\ref{sec:checklist} state the repairs. Colors mark each card's checklist category (Fig.~\ref{fig:checklist}).}
\label{fig:setup}
\end{figure}
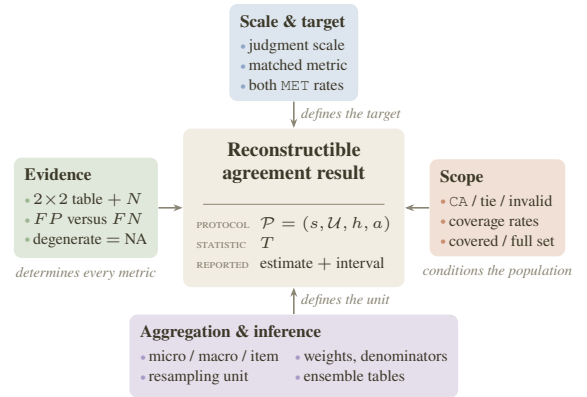

Rubric-based LLM judges are validated by comparing their verdicts with human labels, and the resulting agreement number often decides which judge an evaluator adopts. In one public judge cascade, reconstructed accuracy is $0.874$ if abstentions are excluded, about $0.73$ if they are recoded, and $0.534$ once abstention counts as a third verdict \citep{jung2025trust}. All three numbers describe the same predictions under different handling rules, and pooling rubric decisions instead of averaging per-criterion scores would move the number again. Such choices look like evaluation minutiae and are often left implicit, but they determine the question that the final number answers. Two evaluations reporting $0.87$ and $0.73$ agreement can therefore describe the same judge on the same predictions.

We ask what the reported number is actually measuring once abstentions, invalid outputs, and multiple rubric criteria have been resolved, and how far it speaks for the full evaluation set. The agreement statistics themselves, whose lineage we review in Sec.~\ref{sec:related}, are well understood, so we concentrate on the steps that turn raw judge verdicts into the contingency table from which any such statistic is computed.

Our primary contribution is applying previously known but overlooked statistical results from psychometrics, econometrics, and the corpus-annotation literature to the emerging practice of LLM-judge evaluation, especially with rubrics; we attribute each result where we use it. These results settle choices that evaluators make routinely but rarely state. We define human--LLM agreement at the protocol level, making explicit which side may abstain and how decisions are pooled across criteria and items (Sec.~\ref{sec:protocol}). We restate the binary-metric identities and the $\kappa$--$\phi$ normalization in the form a rubric evaluation needs, with their sampling behavior under matched marginals (Sec.~\ref{sec:binary}). We compare exclusion, negative recoding, and three-class evaluation when either side cannot assess an item, and apply the standard worst-case identification bound to full-set accuracy (Sec.~\ref{sec:abstention}); we have not found that comparison, or the bound, in prior LLM-judge evaluation. We then reanalyze a published self-preference study, an abstaining judge cascade, and a rubric benchmark whose human labels are per-criterion (Secs.~\ref{subsec:case-selfpref},~\ref{subsec:case-cascade}, and~\ref{subsec:case-rubric}), and organize aggregation and multi-judge results into a reporting checklist (Secs.~\ref{sec:aggregation} and~\ref{sec:checklist}). In the cascade, the handling rule alone moves accuracy by $34$ points, more than the judge-to-judge gaps such numbers usually adjudicate; on the rubric benchmark, crossing the four protocol choices moves it by $34.8$ points on one judge's own verdicts. We provide the cascade reconstruction in Appendix~\ref{app:public-abstention} and the full per-criterion rubric benchmark table in Appendix~\ref{app:rubric-table}.

\section{Related Work}
\label{sec:related}

Systems that use an LLM as a judge now anchor evaluation for chat assistants, generated text, and retrieval-augmented generation \citep{zheng2023judging,liu2023geval,gu2024survey}, and their verdicts are validated against human labels before the judge replaces human annotation. Reporting practice varies widely. Our coding of 24 recent LLM-as-judge papers by judged output scale (Tables~\ref{tab:audit-a}--\ref{tab:audit-c}) finds 11 graded or continuous evaluations, most reporting correlations against human ratings, 7 pairwise-preference evaluations with varying tie conventions, 3 binary-verdict evaluations, 2 judge ensembles, and 1 scale-matched mix.\footnote{The single scale-matched entry \citep{bavaresco2025llms} pairs $\kappa$ with categorical datasets and Spearman with graded ones.} Several of the pairwise papers state a tie convention, but none frames tie handling as part of the evaluation target.

\paragraph{Scope and procedure.}
This literature scan is descriptive and non-exhaustive; it is not a systematic review or an estimate of field-wide prevalence. The corresponding author found the 24 papers in Tables~\ref{tab:audit-a}--\ref{tab:audit-c} through Google Scholar searches using terms related to ``LLM Judge'' and ``LLM Judge Rubrics.'' No claim of a systematic screening protocol is made. Each paper was coded for judged output scale, reported metrics and reporting practice, and relevance to the analyses in this paper. The corresponding author performed all coding. No second coder independently checked the entries, no adjudication was conducted, and no inter-coder reliability statistic is available. The tables should therefore be read as illustrative literature examples and boundary cases. Each row corresponds to one paper, identified by its principal system or paper family. The ``redundant correlation coefficients'' noted in the relevance column refer to the binary identity of Sec.~\ref{sec:binary}: on non-degenerate binary vectors, Pearson's $r$, Spearman's $\rho$, Kendall's $\tau_b$, $\phi$, and MCC coincide.

\begin{table*}[t]
\centering
\small
\begin{tabular}{@{}r p{3.1cm} p{3.7cm} p{4.4cm} p{4.3cm}@{}}
\toprule
\textbf{\#} & \textbf{Paper / system} & \textbf{Judged output scale} & \textbf{Metrics / reporting practice} & \textbf{Relevance to this paper} \\
\midrule
1 & G-Eval \citep{liu2023geval} & Graded NLG scores & Reports rank/correlation alignment with human judgments. & \cat{Graded} Multiple correlations are on graded scores, not raw binary verdicts. \\
2 & LLM-Eval \citep{lin2023llmeval} & Numeric multi-dimensional conversation scores & Reports Pearson/Spearman-style score correlations. & \cat{Graded} Not a binary setting. \\
3 & SEEval \citep{wu2023seeval} & Real-valued NLG quality scores & Reports Pearson/Spearman or Spearman/Kendall depending on benchmark. & \cat{Graded} Not a binary setting. \\
4 & CheckEval \citep{lee2025checkeval} & Decomposed binary checklist questions aggregated to ratings & Uses binary checklist questions, then reports aggregate score correlations. & \cat{Binary} The binary equivalence applies to raw yes/no checklist vectors, not aggregated scores. \\
5 & MT-Bench / Chatbot Arena \citep{zheng2023judging} & Pairwise A/B/tie and single-answer ratings & Reports judge\textendash{}human agreement, with and without ties. & \cat{Pairwise} Relevant for tie handling; binary equivalence applies only after excluding or collapsing ties. \\
6 & AlpacaEval / LC-AlpacaEval \citep{dubois2024lengthcontrolled} & Pairwise preference / win rate & Reports win-rate and length-controlled win-rate measures. & \cat{Pairwise} No redundant correlation coefficients. \\
7 & Auto-J \citep{li2023autoj} & Pairwise A/B/tie plus rating and critique & Pairwise preference can include ties. & \cat{Pairwise} Tie handling is explicit, but not framed as selective prediction or \CA. \\
8 & PandaLM \citep{wang2023pandalm} & Pairwise preference judging & Reports $F_1$-style agreement relative to stronger judges or human preferences. & \cat{Pairwise} No redundant Pearson/Spearman/Kendall coefficients found. \\
\bottomrule
\end{tabular}
\caption{Descriptive scan of LLM-as-judge and LLM-evaluator papers, coded by judged output scale, reported metrics, and relevance to this paper's analysis (papers 1--8).}
\label{tab:audit-a}
\end{table*}

\begin{table*}[t]
\centering
\small
\begin{tabular}{@{}r p{3.1cm} p{3.7cm} p{4.4cm} p{4.3cm}@{}}
\toprule
\textbf{\#} & \textbf{Paper / system} & \textbf{Judged output scale} & \textbf{Metrics / reporting practice} & \textbf{Relevance to this paper} \\
\midrule
9 & JudgeLM \citep{zhu2023judgelm} & Pairwise and classification-like judge settings & Reports agreement/accuracy-style measures. & \cat{Pairwise} No redundant correlation coefficients found. \\
10 & LLMBar \citep{zeng2023llmbar} & Pairwise instruction-following contrast & Evaluator chooses the instruction-following output from a pair. & \cat{Pairwise} Motivates scale-aware metrics. \\
11 & Prometheus \citep{kim2023prometheus} & Rubric-based direct assessment and pairwise evaluation & Pearson on graded rubric scores; accuracy on pairwise benchmarks. & \cat{Graded} Mixed setting; direct assessment and binary/pairwise validation are distinct. \\
12 & Prometheus 2 \citep{kim2024prometheus} & Direct assessment plus pairwise ranking & Correlations for direct-assessment scores; agreement/accuracy for pairwise settings. & \cat{Graded} Boundary case: multiple correlations are mostly on graded scores. \\
13 & FLASK \citep{ye2024flask} & Fine-grained 1\textendash{}5 skill scores & Reports Pearson/Spearman/Kendall correlations. & \cat{Graded} Multi-correlation example, but not binary. \\
14 & LLM-Rubric \citep{hashemi2024llmrubric} & Rubric-question response distributions and 1\textendash{}4 satisfaction prediction & Uses rubric dimensions and predicts overall human satisfaction. & \cat{Graded} Not a clean binary verdict setting in the reported overall evaluation. \\
15 & ARES \citep{saadfalcon2023ares} & RAG judge classifiers for context relevance, faithfulness, and relevance & Uses classifier-style evaluation and PPI confidence machinery. & \cat{Binary} Confusion-matrix and precision/recall reporting matter. \\
16 & RAGAS \citep{es2023ragas} & RAG faithfulness/relevance metrics & Reports agreement/validation-style quantities for RAG dimensions. & \cat{Binary} ``Insufficient information'' appears in prompts, but not as a reported \CA{} estimand. \\
\bottomrule
\end{tabular}
\caption{Descriptive scan of LLM-as-judge and LLM-evaluator papers, coded by judged output scale, reported metrics, and relevance to this paper's analysis (papers 9--16, continued).}
\label{tab:audit-b}
\end{table*}

\begin{table*}[t]
\centering
\small
\begin{tabular}{@{}r p{3.1cm} p{3.7cm} p{4.4cm} p{4.3cm}@{}}
\toprule
\textbf{\#} & \textbf{Paper / system} & \textbf{Judged output scale} & \textbf{Metrics / reporting practice} & \textbf{Relevance to this paper} \\
\midrule
17 & PoLL \citep{verga2024replacing} & Panel/jury of LLM evaluators & Evaluates panels across multiple judge settings and datasets. & \cat{Ensemble} Multi-judge motivation; does not show panels as a production default. \\
18 & ChatEval \citep{chan2023chateval} & Multi-agent referee team & Multi-agent discussion followed by scoring/judgment. & \cat{Ensemble} Scale depends on task. \\
19 & JUDGE-BENCH \citep{bavaresco2025llms} & Multiple datasets with categorical and graded annotations & Cohen's $\kappa$ for categorical datasets, Spearman for graded datasets, Krippendorff's $\alpha$ for human agreement. & \cat{Scale-matched} Positive example of scale-aware metric selection. \\
20 & FineSurE \citep{song2024finesure} & Fine-grained summarization dimensions & Multi-dimensional summarization evaluation. & \cat{Graded} More relevant to scale/criteria framing than to binary redundancy. \\
21 & TIGERScore \citep{jiang2023tigerscore} & Trained explainable text-generation metric & Correlation-style meta-evaluation with human ratings. & \cat{Graded} Not binary. \\
22 & InstructScore \citep{xu2023instructscore} & Score plus diagnostic feedback & Correlates generated scores with human ratings. & \cat{Graded} Not binary. \\
23 & InstructEval \citep{chia2023instructeval} & Holistic evaluation of instruction-tuned LLMs & Includes Likert-style writing evaluation and task metrics. & \cat{Graded} Mostly not a binary judge-verdict setting. \\
24 & JuStRank \citep{gera2024justrank} & Pairwise/system-ranking judgments & Studies LLM judges as system rankers. & \cat{Pairwise} Relevant especially after aggregation; no redundant correlation coefficients found. \\
\bottomrule
\end{tabular}
\caption{Descriptive scan of LLM-as-judge and LLM-evaluator papers, coded by judged output scale, reported metrics, and relevance to this paper's analysis (papers 17--24, continued). The small-caps value in the last column is a primary relevance tag, not an exhaustive or mutually exclusive category.}
\label{tab:audit-c}
\end{table*}

Cohen's $\kappa$ \citep{cohen1960kappa}, the phi coefficient and its identity with MCC \citep{matthews1975comparison}, the association coefficients for $2\times2$ tables and their marginal dependence \citep{warrens2008association,warrens2014marginal}, the multi-rater statistics and their integration \citep{scott1955reliability,fleiss1971measuring,krippendorff2004content,conger1980integration}, large-sample variances \citep{fleiss1969large,blackman1993estimating}, and the prevalence problems of $\kappa$ \citep{feinstein1990high} all predate LLM evaluation by decades. The results we apply are likewise established; only the $\alpha$--$\kappa_F$ relation of Sec.~\ref{subsec:ensembles} has already reached computational linguistics, through the annotation-agreement survey of \citet{artstein2008intercoder}. These results bear directly on which agreement number to report, yet the LLM-as-judge papers in our scan neither cite nor apply them.

Work questioning agreement reporting for LLM judges is recent. \citet{elangovan2025beyond} show that one aggregate correlation obscures systematic machine--human differences once human label uncertainty is accounted for, and \citet{chehbouni2025neither} apply measurement theory to argue that judge validity is assumed rather than established. \citet{binette2024estimands} adapt the clinical-trials estimands framework to AI evaluation and exhibit rank reversals under a mis-specified target; we specialize that framing to human--judge agreement in Sec.~\ref{subsec:estimands}. Closest to our setting, \citet{guerdan2025validating} treat rating indeterminacy as a validation-design choice and relate judge-performance measures across agreement metrics, which is the move we make for abstention; we differ in scope, treating the scale, selection, handling, and aggregation choices jointly. \citet{jung2025trust} calibrate judge abstention and escalation to trade coverage for lower error, and \citet{lee2026howto} correct downstream scores using judge sensitivity and specificity. These two design better judges or correct their outputs; our analysis concerns the agreement claim itself.

\section{Agreement as a Measurement Protocol}
\label{sec:protocol}

\begin{figure}[t]
\centering
\includegraphics[width=\columnwidth]{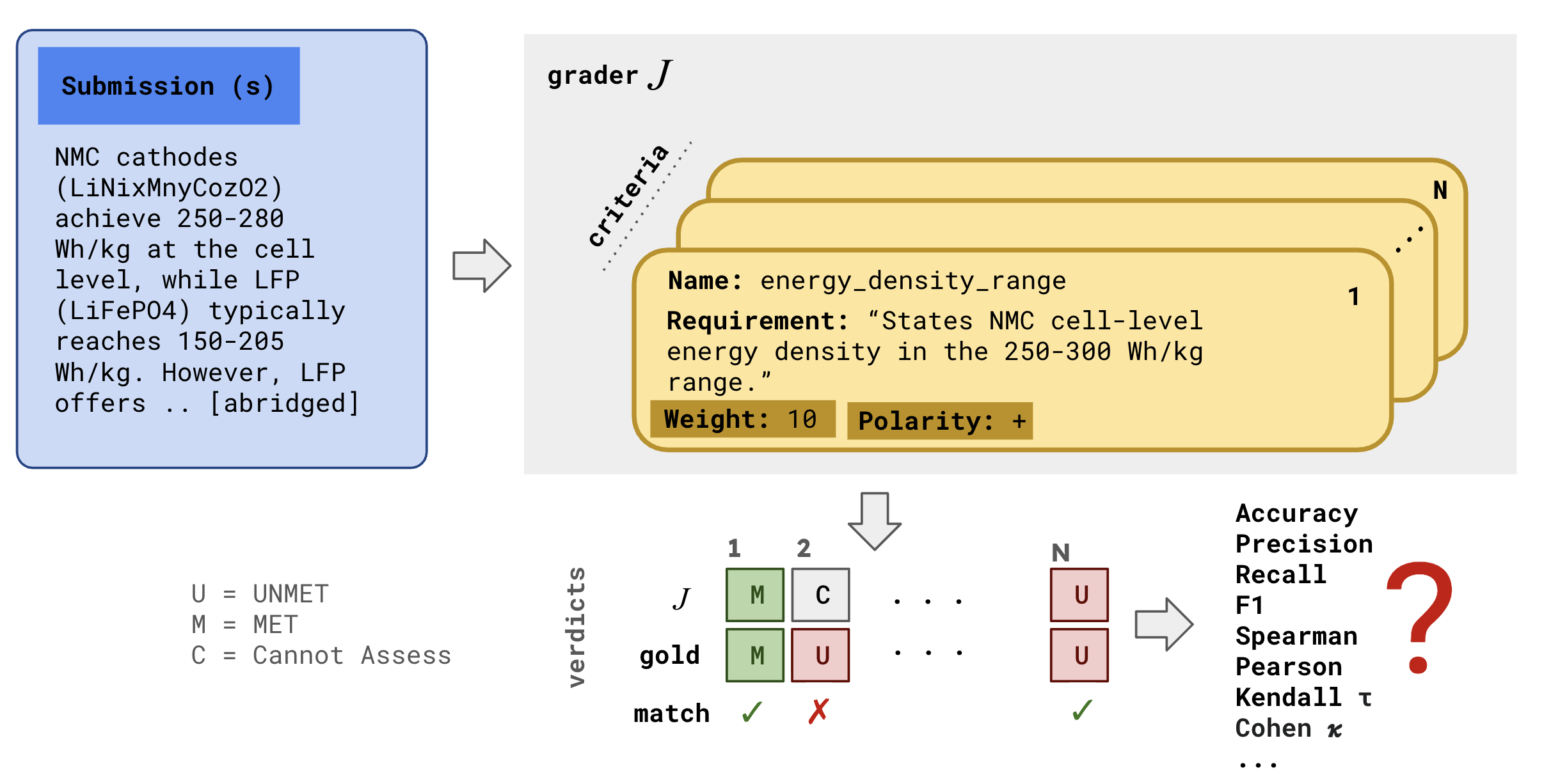}
\caption{Rubric-based LLM judging. An LLM judge $J$ (the grader in the figure) grades a submission against the criteria of a rubric, each stating a name, a requirement, a weight, and a polarity, and emits one verdict per criterion: \MET{} (M), \UNMET{} (U), or \CANNOTASSESS{} (C). Comparing these verdicts with a human annotator's gold verdicts yields the agreement numbers whose choice, meaning, and reporting this paper examines; the energy-density criterion shown is the running example of Sec.~\ref{sec:protocol}.}
\label{fig:rubric-judging}
\end{figure}

We begin with rubrics with binary criteria and the verdict table of which every metric in this paper is an explicit function. Consider validating a rubric judge against human gold labels, one binary verdict per criterion per submission (Fig.~\ref{fig:rubric-judging}). Let $y_i\in\{0,1\}$ be the human label and $\hat y_i$ the judge verdict for decision $i=1,\ldots,N$, with $1$ denoting \MET{} and $0$ \UNMET{}; for a criterion asking whether a submitted battery specification states a per-cell energy density within a required range, $\hat y_i=1$ records that the judge found one. Their complete comparison is the $2\times2$ verdict table with cells $TP=n_{11}$, $FN=n_{10}$, $FP=n_{01}$, and $TN=n_{00}$, where $n_{ab}=|\{i:y_i=a,\hat y_i=b\}|$ counts decisions with human label $a$ and judge verdict $b$, and $N = TP + FN + FP + TN$. We write $\pi = (TP + FN)/N$ for the human \MET{} rate, $\hat{\pi} = (TP + FP)/N$ for the judge \MET{} rate, and $p_o = (TP+TN)/N$ for accuracy, the exact-match rate.

Each metric summarizes a different part of this table. Accuracy, $F_1$, $\kappa$, and $\phi$/MCC are all unchanged when $FP$ and $FN$ are exchanged, while precision and recall swap values; only the full confusion matrix distinguishes a judge that misses true positives from one that over-predicts. Accuracy and $\kappa$ summarize exact matches, $\phi$/MCC and the correlations summarize association, and precision, recall, and $F_1$ describe one class. All compare the judge with a single human reference, so they provide no evidence about inter-judge reliability or broader construct validity. Throughout, the table is non-degenerate: all four marginals $TP{+}FN$, $FN{+}TN$, $TP{+}FP$, and $FP{+}TN$ are positive, so every denominator in the identities below is nonzero. In explicit form, $F_1$ describes only the positive class: with precision $P = TP/(TP + FP)$ and recall $R = TP/(TP + FN)$,
\begin{equation}
\label{eq:f1}
F_1 = \frac{2\,P\,R}{P + R} = \frac{2\,TP}{2\,TP + FP + FN},
\end{equation}
which does not depend on $TN$ and changes when the positive and negative labels are exchanged; the negative-class analog is $F_1^{(\UNMET)} = 2\,TN/(2\,TN + FP + FN)$, and specificity is $TN/(TN+FP)$.
\begin{takeaway}
Always show the per-criterion $2\times2$ tables alongside the metrics computed from them. Any metric can be recomputed from a stored table, but the table cannot be recovered from reported metrics alone.
\end{takeaway}

\subsection{Judgment Scales}
\label{subsec:scale}

Rubric-based LLM judges generalize beyond binary verdicts. Five judgment scales recur in LLM-judge validation: binary \MET/\UNMET{} verdicts, pairwise preference with ties, ordinal ratings, continuous quality judgments, and nominal verdicts with abstention. Table~\ref{tab:scale} names the five scales and the statistics each admits; the equivalence results of Sec.~\ref{sec:binary} concern the binary scale alone. An ordinal example shows how much the choice of statistic alone can move a reported number. Consider a judge that rates one point above the human on every item of a 1--5 ordinal scale, $y=(1,2,3,4)$ and $\hat y=(2,3,4,5)$. It has perfect Pearson, Spearman, and Kendall association but zero exact agreement; unweighted, linear-weighted, and quadratic-weighted $\kappa$ are $-3/13$, $1/3$, and $5/7$. Whether this judge appears worse than chance or perfect depends on the choice of the statistic alone. A practitioner must therefore fix the scale and its matched statistic before computing anything.
\begin{table*}[t]
\centering
\small
\begin{tabular}{@{}p{3.0cm} p{3.2cm} p{8.2cm}@{}}
\toprule
\textbf{Judgment scale} & \textbf{Examples} & \textbf{Applicable statistics} \\
\hline
Binary verdict & \MET/\UNMET{} criterion decisions & Confusion matrix and marginals; accuracy, class-conditional metrics, or $\kappa$ per target; at most one of $\phi$/MCC/Pearson/Spearman/Kendall's $\tau_b$ \\
Pairwise preference & A wins / B wins / tie & An explicit tie-handling convention; after reduction to a binary outcome, the binary-identity results of Sec.~\ref{sec:binary} \\
Ordinal value & 1\textendash{}5 rubric ratings & Weighted $\kappa$ for category agreement under declared distance weights; Spearman and Kendall for rank association, which differs from exact agreement \\
Continuous value & Real-valued quality judgments & Pearson for linear association, Spearman for monotone association, Kendall for pair concordance; a scale-sensitive error measure when agreement in level matters \\
Nominal with abstention & \MET/\UNMET/\CA & A three-class agreement statistic with an explicit weight matrix, or a stated reduction mode for collapsing to two classes \\
\hline
\end{tabular}
\caption{A judgment-scale taxonomy for LLM-as-judge evaluation.}
\label{tab:scale}
\end{table*}

\begin{takeaway}
A rubric judge often emits binary criterion verdicts that a weighted rule later combines into an item score. The verdicts and the score sit on different judgment scales and call for different statistics; a report should state whether its number describes the criterion verdicts or the item score.
\end{takeaway}

\subsection{Protocols and Estimands}
\label{subsec:estimands}

Beyond the choice of the judgment scale, the practitioner makes three further choices before computing any statistic: which subset of the data to select, how to resolve exceptions such as abstentions and invalid generations, and how to pool or weight verdicts across criteria and items. We collect the four choices into a \emph{measurement protocol} $\mathcal{P}=(s,\mathcal{U},h,a)$ (Fig.~\ref{fig:setup}), where $s$ maps outputs to a judgment scale, $\mathcal{U}$ is the selected population of cases, $h$ specifies the handling rule for exceptions, and $a$ specifies the aggregation rule. When the judge, the human, or both may abstain, for example, $s$ could be a binary or three-class representation and $h$ could exclude, recode, or retain a cannot-assess (\CA) verdict. Under recoding or retention, $\mathcal{U}$ contains every case; under exclusion, it contains only the cases in which both the human and the LLM judge returned a \MET{} or \UNMET{} verdict. A judge abstains when the submission lacks the evidence a criterion asks about; a human annotator may abstain for the same reason or because the criterion does not apply to the submission. In our experience, current LLM-judge evaluations rarely handle these outcomes systematically: prompts often permit an ``insufficient information'' answer \citep{es2023ragas} and pairwise protocols record ties \citep{zheng2023judging,li2023autoj}, but the handling rule and its effect on the reported number are seldom stated.

The protocol fixes the \emph{estimand}, the quantity that the reported number estimates. Applying $\mathcal{P}$ to the raw outputs yields either a single joint distribution $P_{\mathcal{P}}$ over paired human--judge values or, when criteria are scored separately and averaged with weights $w_g$, one distribution $P_{\mathcal{P},g}$ per criterion $g$. For statistic $T$, the estimand is
\begin{equation}
\label{eq:protocol-estimand}
\theta_{T,\mathcal{P}}=
\begin{cases}
T(P_{\mathcal{P}}), & \text{single distribution},\\
\sum_g w_gT(P_{\mathcal{P},g}), & \text{per-criterion average},
\end{cases}
\end{equation}
and replacing each population distribution with its empirical table gives the estimate $\hat\theta_{T,\mathcal{P}}$. The two subscripts separate the choice of statistic from the choice of protocol. Evaluation reports name the statistic as a matter of course but often leave the protocol implicit, and the same statistic computed under different protocols estimates different quantities; a result is therefore reconstructible only when both are reported.
\begin{takeaway}
Decide the four protocol components (the judgment scale, the selected population, the handling rule for exceptions, and the aggregation rule) before looking at any judge outputs, and evaluate every judge under comparison with the same protocol. A handling rule chosen after the results are in can be the one that gives the judge its highest agreement number. When two judges are scored under different protocols, their numbers estimate different quantities, and the difference between them reflects the protocols as much as the judges.
\end{takeaway}

\section{Identities and Distinctions Among Binary Agreement Metrics}
\label{sec:binary}

Of the estimand's two coordinates, the statistic $T$ and the protocol $\mathcal{P}$, we examine the statistic first. For rubrics with binary criteria, the choice among the familiar metrics of Sec.~\ref{sec:protocol} is smaller than it appears: the five agreement metrics are one statistic, Cohen's $\kappa$ differs from them only through its normalization, and under matched \MET{} rates even the asymptotic sampling variances coincide.

\subsection{Five Metrics, One Statistic}
\label{subsec:equiv}

A validation report that includes several of the five metrics---Pearson's $r$, Spearman's $\rho$, Kendall's $\tau_b$, $\phi$, and the Matthews correlation coefficient (MCC)---invites reading their agreement as corroborating evidence.

\begin{fact}[\citealp{matthews1975comparison}; \citealp{kendall1945treatment}; \citealp{warrens2008association}]\label{obs:equiv}
For any pair of binary vectors $(y, \hat{y}) \in \{0,1\}^N$ in which each vector contains both labels,
\begin{equation}
\label{eq:obs1}
\rho_{\text{Pearson}} = \rho_{\text{Spearman}} = \tau_b = \phi = \text{MCC}.
\end{equation}
\end{fact}
A second coefficient from the list adds no information on binary verdicts, and the identity is elementary to verify. For binary vectors the Pearson correlation reduces to the phi coefficient \citep{matthews1975comparison},
\begin{equation}
\label{eq:phi}
\phi = \frac{TP \cdot TN - FP \cdot FN}{\sqrt{(TP{+}FP)(TP{+}FN)(TN{+}FP)(TN{+}FN)}},
\end{equation}
which is identical to the binary Matthews correlation coefficient.

\paragraph{Spearman.} Spearman's $\rho$ is Pearson's $r$ computed on the rank-transformed inputs. With only two distinct values present, all zeros receive the common average rank $a_0 = (n_0 + 1)/2$ and all ones receive $a_1 = n_0 + (n_1 + 1)/2$, where $n_0, n_1$ are the counts of zeros and ones. The map $x \mapsto a_0 + (a_1 - a_0)\,x$ is affine, and Pearson's $r$ is invariant under affine rescaling of either input, so $\rho_{\text{Spearman}} = \rho_{\text{Pearson}}$.

\paragraph{Kendall's $\tau_b$.} The tie-corrected form \citep{kendall1945treatment} is
\begin{equation}
\label{eq:taub}
\tau_b = \frac{C - D}{\sqrt{(n_{\text{tot}} - n_y)(n_{\text{tot}} - n_{\hat{y}})}},
\end{equation}
where $n_{\text{tot}}=\binom{N}{2}$ and $n_y,n_{\hat y}$ count pairs tied on each vector. In the $2\times2$ table, concordant cross-cell pairs number $C=TP\cdot TN$ and discordant pairs number $D=FN\cdot FP$. The untied-pair counts are $(TP+FN)(FP+TN)$ for $y$ and $(TP+FP)(FN+TN)$ for $\hat y$. Therefore,
\begin{equation}
\label{eq:taub2}
\begin{aligned}
\tau_b &= \frac{TP \cdot TN - FP \cdot FN}{\sqrt{(TP+FN)(FP+TN)}} \\[2pt]
&\quad\times \frac{1}{\sqrt{(TP+FP)(FN+TN)}} = \phi.
\end{aligned}
\end{equation}
\begin{takeaway}
Report only one of Pearson, Spearman, Kendall's $\tau_b$, $\phi$, and MCC on binary verdicts; they are interchangeable there, and $\phi$/MCC is the most direct. If two of them differ in a pipeline's output, suspect a degenerate criterion or an implementation error rather than judge behavior.
\end{takeaway}

\subsection{Marginal Normalization: Kappa Versus Phi}
\label{subsec:kappaphi}

Cohen's $\kappa$ falls outside the equivalence class. It discounts the agreement expected from the marginal rates alone, $\kappa = (p_o - p_e)/(1 - p_e)$ with $p_e = \pi\hat{\pi} + (1-\pi)(1-\hat{\pi})$ \citep{cohen1960kappa}. Multiplying $\kappa$ through by $N^2$ and canceling shared terms yields the count form
\begin{equation}
\kappa = \frac{2\,(TP \cdot TN - FP \cdot FN)}{(TP{+}FN)(FN{+}TN) + (TP{+}FP)(FP{+}TN)},
\end{equation}
which shares the numerator of $\phi$ in Eq.~\eqref{eq:phi} up to a factor of two. Writing $A = (TP+FN)(FN+TN)$ and $B = (TP+FP)(FP+TN)$, the two denominators are $A+B$ and $2\sqrt{AB}$, and the AM--GM inequality $A + B \ge 2\sqrt{AB}$, with equality iff $A = B$, yields the multiplicative identity
\begin{equation}
\label{eq:ratio}
\kappa = q(\pi,\hat{\pi})\,\phi, \qquad
q(\pi,\hat{\pi}) = \frac{2\sqrt{AB}}{A + B} \in (0,1],
\end{equation}
so $|\kappa| \le |\phi|$. Here $A=B$, and with it $q=1$, exactly when the judge and human \MET{} rates coincide ($\pi=\hat\pi$), so $|\kappa|=|\phi|$ when those rates match or both metrics are zero. Otherwise $q$ quantifies how much marginal mismatch attenuates $\kappa$ relative to $\phi$: a judge that orders the items exactly as the human does but commits to a different \MET{} rate receives a strictly lower $\kappa$ than $\phi$. Equation~\eqref{eq:ratio} is due to \citet{sahu2024linear} and extends a longer line on the marginal dependence of the $2\times2$ $\kappa$ \citep{warrens2008association,warrens2014marginal}; we restate it because it turns a familiar complaint about $\kappa$ into a number an evaluator can compute from the table already in hand.

The rescaled-accuracy form of $\kappa$ also explains a failure mode familiar from rubrics with easy criteria. In the balanced case $\pi = \hat{\pi} = 0.5$, $p_e = 0.5$ and $\kappa = 2\,p_o - 1$. For a criterion that nearly every response meets, both marginals approach $1$, $p_e \to 1$, and high raw agreement can then coexist with much lower chance-corrected agreement, a form of the \emph{kappa paradox} discussed by \citet{feinstein1990high}.\footnote{Alternative chance corrections substitute different baselines \citep{brennan1981coefficient,gwet2008computing,byrt1993bias}; none removes the need for the marginals and the contingency table.}
\begin{takeaway}
Before reading a low $\kappa$ as a bad judge, check the \MET{} rates: on rubric criteria that nearly all submissions meet or fail, $\kappa$ is attenuated by construction. Compare $\kappa$ across criteria only when their \MET{} rates are comparable, and report the human and judge rates with every $\kappa$.
\end{takeaway}

\subsection{Sampling Variance under Matched Marginals}
\label{subsec:variance}

Sampling variance completes the comparison: for multinomial sampling of a non-degenerate table with matched \MET{} rates, $\hat\kappa$ and $\hat\phi$/MCC also share the same asymptotic variance, a consequence of $q$ being stationary at $\pi=\hat\pi$. To derive it, let $\mathbf{p}=(p_{11},p_{10},p_{01},p_{00})^\top$ denote the table's population cell-probability vector. Both $\kappa$ and $\phi$ are ratios of polynomials in $\mathbf{p}$ whose denominators are nonzero at any non-degenerate table, so both are continuously differentiable in a neighborhood of $\mathbf{p}$ and the delta method applies. On the matched-marginals slice $p_{10}=p_{01}$, both metrics reduce to the same function, $\kappa=\phi=1-p_{10}/(p_{1\cdot}p_{0\cdot})$. Their gradients therefore agree along the two tangent directions of that slice. The slice-orthogonal direction $(0,+1,-1,0)$ swaps $p_{10}$ and $p_{01}$; both metrics are symmetric and stationary under this swap. Thus, the two four-dimensional gradients can differ only along the simplex normal $\mathbf{1}$. The multinomial covariance
\begin{equation}
\label{eq:multinomial-covariance}
\Sigma=N^{-1}\!\left(\operatorname{diag}(\mathbf{p})-\mathbf{p}\mathbf{p}^\top\right)
\end{equation}
satisfies $\Sigma\mathbf{1}=0$, so the delta-method quadratic forms $\nabla\kappa^\top\Sigma\nabla\kappa$ and $\nabla\phi^\top\Sigma\nabla\phi$ coincide. This gives the claimed equality $N\operatorname{AVar}(\hat{\kappa}) = N\operatorname{AVar}(\hat{\phi})$. Because both estimators are ratios of multinomial counts, the statement is about their limiting distribution: at finite $N$ a marginal can be degenerate with positive probability, leaving $\hat\phi$ undefined.

Matched marginals means that the judge and the human assign \MET{} equally often; it does not mean that they agree item by item. The equality is also pointwise, offering no guarantee when the two \MET{} rates differ. In rubric evaluation the unit of analysis is rarely an isolated criterion decision: criteria are clustered within items (multiple criteria per submission), within prompts or rubrics (multiple items sharing a rubric), and within judges (in an ensemble, multiple decisions by the same model). A flat decision-level resample understates uncertainty when within-cluster correlations are positive, as is the norm when a single judge mistake propagates across the criteria of one item. In practice an item-level or hierarchical cluster bootstrap (resample items, then optionally criteria within items) supports judge comparisons, with undefined replicates reported; standard implementations are available in the survey-statistics literature \citep{field2007bootstrapping}.

As a numerical check of the equality, cell probabilities $(0.4,0.1,0.1,0.4)$ give $\kappa=\phi=0.6$ and $\operatorname{Var}(\hat\kappa)=\operatorname{Var}(\hat\phi)=0.640/N$ under the multinomial delta method \citep{fleiss1969large}; treating these binary values as continuous produces the incorrect $0.41/N$.

A further numerical check locates the equality's boundary. Off matched marginals the variances diverge by a second-order term: at $(0.4, 0.1, 0.2, 0.3)$, where $\pi = 0.5$, $\hat{\pi} = 0.6$, $\kappa = 0.400$, and $\phi = 1/\sqrt{6} \approx 0.408$, the delta method gives $\operatorname{Var}(\hat\kappa) = 504/(625N) \approx 0.806/N$ and $\operatorname{Var}(\hat\phi) = 13/(16N) \approx 0.813/N$; the variance gap is second-order in the marginal mismatch, mirroring the pointwise gap.
\vadjust pre{\penalty-10000}%
\begin{takeaway}
Rank judges only with an uncertainty estimate attached; at $N=300$ binary decisions the delta-method standard error of $\hat\kappa$ falls between $0.04$ and $0.06$ for most tables, so gaps of a few points are within sampling noise. Resample whole items rather than individual verdicts, and avoid variance formulas built for continuous values.
\end{takeaway}

\subsection{Identities and Distinctions in a Published Evaluation}
\label{subsec:case-selfpref}

A metric can be structurally unable to detect the failure an evaluator is looking for. As a case study, consider the paper of \citet{llm-judge-self-bias-watoka} on self-preference bias, the tendency of a judge to favor its own generations \citep{llm-judge-self-bias-panickssery}.\footnote{We are not singling out this work; it is a well-cited example chosen for illustration. The source excludes format-noncompliant outputs but does not report how Chatbot Arena ties were handled in this binary table.} The paper reports a human--GPT-4 win/lose table (their Fig.~2), derived from position-swap-averaged judge probabilities, for GPT-4 judging comparisons that involve its own outputs, and bases its conclusions on the two class recalls; the statistics we present are our reanalysis. Taking ``GPT-4's own output wins'' as positive gives the verdict table $TP=1{,}852$, $FN=108$, $FP=160$, $TN=118$ over $N=2{,}238$ comparisons: accuracy $0.880$ and $F_1=0.933$, against $0.876$ and $0.934$ for a baseline that always declares GPT-4 the winner. By accuracy and by $F_1$, which ignores $TN$, the judge is indistinguishable from that baseline. The agreement metrics separate them: Pearson's $r$, Spearman's $\rho$, Kendall's $\tau_b$, $\phi$, and MCC are each $0.404$, as Fact~\ref{obs:equiv} requires, and $\kappa=0.402$; the small gap follows from the nearly matched marginals of Eq.~\eqref{eq:ratio}. The baseline has $\kappa=0$ and undefined $\phi$, while the judge's specificity on human-labeled losses is $0.424$; the failure is concentrated on a class that is $12.4\%$ of the data and invisible to $F_1$. The source's own bias measure is the gap between the two class recalls, $0.945$ against $0.424$, about $0.52$, which it frames as an equal-opportunity difference; that contrast is class-conditional and answers a different question from the agreement metrics above. Each reading summarizes the same $2{,}238$ verdicts; the disagreement comes entirely from which cells of the table each metric uses.

\paragraph{Significance analysis.}
The judge and the always-positive baseline are compared on the same $2{,}238$ paired decisions, so paired accuracy is tested with McNemar's exact test \citep{mcnemar1947note} on the discordant counts $(118, 108)$: $p=0.549$, and accuracy does not distinguish the judge from the baseline. The agreement metrics separate them with uncertainty attached: resampling the $2{,}238$ decisions with replacement ($10^4$ replicates, none undefined) gives percentile intervals of $[0.346, 0.463]$ for the judge's $\phi=0.404$ and $[0.343, 0.460]$ for $\kappa=0.402$; both intervals exclude the baseline's structural $\kappa=0$.

\section{Handling Abstention}
\label{sec:abstention}

The binary analysis assumed that humans and LLM judges always choose \MET{} or \UNMET{} when evaluating an input against a criterion. However, it is good practice---for both humans and LLMs---to abstain from a judgment by producing a \CANNOTASSESS{} (\CA) response to avoid making a forced choice.\footnote{Equivalent labels in the literature include \verdicttoken{IDK}, \verdicttoken{REFUSE}, and \verdicttoken{ABSTAIN}.} A judge cannot assess that energy-density criterion, for example, when the specification states no per-cell density at all. A \CA{} output carries information that \UNMET{} does not, and the practitioner must decide what it represents before computing agreement.

\subsection{Abstention Handling Modes}
\label{subsec:modes}

Let $\mathcal{D} = \{(y_i, \hat{y}_i)\}_{i=1}^N$ with $y_i, \hat{y}_i \in \{\MET, \UNMET, \CA\}$. Three handling rules recur, and each targets a different estimand.

\noindent\textbf{Exclusion:} Introduce latent forced-choice verdicts $y_i^*, \hat{y}_i^* \in \{0, 1\}$ and abstention indicators $m_i, \hat{m}_i \in \{0, 1\}$. The observed value is $y_i=y_i^*$ when $m_i=0$ and $y_i=\CA$ otherwise, symmetrically for the judge. Exclusion discards a pair when either side abstains. For statistic $T$, exclusion targets $T\bigl(\mathbb{P}_{(y^*, \hat{y}^*) \,\mid\, m = 0 \,\wedge\, \hat{m} = 0}\bigr)$, agreement conditional on neither side abstaining; it describes all cases only under additional assumptions taken up in Sec.~\ref{subsec:bounds}. If \CA{} means that the criterion is genuinely inapplicable, with no latent binary verdict behind it, a three-class representation or an explicitly restricted population is the better match.

\noindent\textbf{Negative recoding:} Map \CANNOTASSESS{} to \UNMET{} in both vectors, then compute the binary metric. This is a valid one-versus-rest reduction when the target is \MET{} versus non-\MET{}; it no longer distinguishes \CA{} from \UNMET{}, so the negative label changes meaning. Recoding also makes an abstaining judge and a committing judge indistinguishable: one that abstains on uncertain negatives and one that answers \UNMET{} on the same cases receive identical scores, and under a training objective that rewards calibrated abstention, as in selective prediction, the recoding erases the very signal the objective conditions on.

\noindent\textbf{Three-class retention:} Retain \CANNOTASSESS{} as a distinct category and report a $3 \times 3$ agreement metric such as multi-class accuracy, multi-class Cohen's $\kappa$ \citep{conger1980integration}, or a weighted variant whose weights state how costly each type of disagreement is.

\begin{table}[t]
\centering
\small
\begin{tabular}{@{}lrrrr@{}}
\toprule
\multicolumn{5}{@{}l}{(a) Verdict counts} \\
\midrule
$y \backslash \hat{y}$ & \MET & \UNMET & \multicolumn{2}{r}{Total} \\
\MET & 1{,}003 & 181 & \multicolumn{2}{r}{1{,}184} \\
\UNMET & 420 & 97 & \multicolumn{2}{r}{517} \\
\CA & 242 & 64 & \multicolumn{2}{r}{306} \\
Total & 1{,}665 & 342 & \multicolumn{2}{r}{2{,}007} \\
\midrule
\multicolumn{5}{@{}l}{(b) Statistics under the three handling rules} \\
\midrule
Rule & $N_{\mathrm{eff}}$ & Acc & $\kappa$ & $F_1^{(\MET)}$ \\
Exclusion & 1{,}701 & 0.647 & 0.040 & 0.769 \\
Recoding & 2{,}007 & 0.580 & 0.047 & 0.704 \\
Three-class & 2{,}007 & 0.548 & 0.032 & n/a \\
\bottomrule
\end{tabular}
\caption{The $2{,}007$ rubric decisions of \citet{hashemi2024llmrubric} (most probable option as the verdict, \MET{} $=$ rating $\ge3$) evaluated under three abstention-handling rules; recoding maps \CA{} to \UNMET{}. Only the human abstains, so the judge has no \CA{} column.}
\label{tab:abstention-example}
\end{table}

We apply all three rules to one set of real rubric decisions (Table~\ref{tab:abstention-example}): nine criteria scored on $223$ human--AI dialogues by a human annotator and by an LLM judge \citep{hashemi2024llmrubric}, in which the human may record that a criterion does not apply and the judge is forced to choose. We analyze this benchmark in Sec.~\ref{subsec:case-rubric}; here it supplies the counts. The rules produce different numbers because they change the scored population or the category space. Exclusion conditions on the $1{,}701$ decisions the human scored, whereas recoding scores all $2{,}007$ after merging \CA{} with \UNMET{}; three-class accuracy and $\kappa$ retain the full category space. Binary $F_1^{(\MET)}$ applies only after a binary reduction. Accuracy moves $10$ points across the three rules while $\kappa$ moves by $0.015$: which rule is chosen matters more for one summary than for the other, and neither ordering is a property of the judge.

Recoding carries one exact invariance, specific to \MET-class $F_1$.
\begin{proposition}\label{prop:recoded-f1}
Let $n_{ij}$, $i,j\in\{1,0,C\}$, be a $3\times3$ verdict table with human rows and judge columns, where $1$ denotes \MET, $0$ denotes \UNMET, and $C$ denotes \CA, and form the $2\times2$ table that merges $C$ into $0$ in both vectors. The \MET-class $F_1$ of the merged table equals the one-versus-rest \MET-class $F_1$ of the original table.
\end{proposition}
\noindent\emph{Proof.} Both quantities equal $2n_{11}/(2n_{11}+f_p+f_n)$ with $f_p=n_{01}+n_{C1}$ and $f_n=n_{10}+n_{1C}$, because merging $C$ into $0$ changes neither the $n_{11}$ cell nor which pairs count as false positives or false negatives with respect to \MET. Accuracy and $\kappa$ depend in addition on the merged negative--negative cell, so they have no such invariance.\hfill$\square$

With three labels the binary identity also breaks: Fact~\ref{obs:equiv} relies on each vector having exactly two values. The statistic and the disagreement costs are therefore modeling choices. Fact~\ref{obs:equiv} also fails on binary data when a marginal is constant, but the two failures differ: a degenerate marginal can be repaired by dropping or flagging the criterion, while the three-class break persists under every encoding.
\begin{proposition}\label{prop:separation}
On three-valued vectors the product-moment and rank statistics need not coincide: there exist $y,\hat y\in\{0,1,2\}^6$ on which Pearson's $r$, Spearman's $\rho$, and Kendall's $\tau_b$ are pairwise distinct.
\end{proposition}
\noindent\emph{Proof.} For $y=(0,1,2,2,0,2)$ and $\hat y=(0,2,1,2,2,0)$, direct computation gives $r=-0.0345$, $\rho=-0.0833$, and $\tau_b=-0.0909$; the ancillary code recomputes all three.\hfill$\square$

\noindent Under an ordinal encoding $\{0, \tfrac{1}{2}, 1\}$ for $\{\UNMET,\CA,\MET\}$, the average-rank map is affine only when the \UNMET{} and \MET{} counts are equal, so Spearman and Pearson generally differ. The phi coefficient has no canonical $3\times3$ analog, and multi-class MCC \citep{gorodkin2004comparing} is neither Pearson on an encoding nor Kendall's $\tau_b$. For judge panels (Sec.~\ref{subsec:ensembles}), average pairwise Spearman and $\tau_b$ on three-valued data separate as well; their values depend on the full pairwise three-category contingency tables, including judge-specific marginals and tie structure.
\begin{takeaway}
Choose the handling rule from what \CA{} means in the rubric, not from convenience: exclusion when a hidden \MET/\UNMET{} answer exists, a restricted population or three classes when the criterion can be genuinely inapplicable, recoding only when the target is \MET{} versus everything else. Numbers computed under different handling rules are not comparable, even on the same verdicts (Table~\ref{tab:abstention-example}).
\end{takeaway}

\subsection{Coverage and Bounds on Full-Set Accuracy}
\label{subsec:bounds}

An accuracy of $0.90$ is less informative if the judge abstains on half the cases, which is why selective-prediction papers report performance together with coverage \citep{chow1970optimum,elyaniv2010foundations,kamath2020selective,xin2021abstention,kadavath2022language}. Judge coverage is $\Pr[\hat m=0]$; when the human labels may also abstain, the relevant quantity for exclusion is joint coverage, $\gamma = \Pr[m=0 \wedge \hat m=0]$, the probability that neither side abstains.

Coverage and covered accuracy together still leave the full evaluation set underdetermined. For $0<\gamma<1$, let $A_{\mathrm{cov}}=\Pr(y^*=\hat y^*\mid m=0 \wedge \hat m=0)$ be accuracy on covered cases, and write $A_{\mathrm{uncov}}$ for the unobserved accuracy on the remaining cases. Then
\begin{equation}
\label{eq:accuracy-bounds}
\begin{split}
A_{\mathrm{full}}
&=\gamma A_{\mathrm{cov}}+(1-\gamma)A_{\mathrm{uncov}},\\
A_{\mathrm{full}}
&\in [\gamma A_{\mathrm{cov}},\;\gamma A_{\mathrm{cov}}+1-\gamma].
\end{split}
\end{equation}
The range is sharp: the lower endpoint treats every uncovered pair as a disagreement, the upper endpoint treats every uncovered pair as an agreement, and the width is exactly the uncovered fraction $1-\gamma$. Equation~\eqref{eq:accuracy-bounds} is the worst-case identification bound for a bounded outcome with missing data, specialized to a $0/1$ agreement indicator with abstention as the missingness mechanism \citep{manski1989anatomy,manski2003partial}. The selective-prediction papers cited above pair risk with coverage but, as far as we have found, none states the induced bound on full-set performance, and LLM-judge evaluations report covered accuracy without either.

The interval captures identification uncertainty, what the observed table cannot pin down about how \CA{} cases would have resolved. Sampling uncertainty is separate; a cluster bootstrap is still needed for it.
\begin{takeaway}
Never report covered accuracy without joint coverage: at coverage $\gamma$, full-set accuracy is known only to an interval of width $1-\gamma$, whatever the covered accuracy is. A judge that answers $60\%$ of cases at $0.90$ covered accuracy guarantees a full-set accuracy of only $0.54$.
\end{takeaway}

\subsection{Exact Finite-Allocation Bounds}
\label{subsec:finite-allocation}

The closed form of Eq.~\eqref{eq:accuracy-bounds} is specific to accuracy; an exact empirical bound for any binary-table metric follows by enumeration. Let $n_{ij}$ be an observed $3\times3$ table with human rows and judge columns indexed by $i,j\in\{0,1,C\}$, where $C$ denotes \CA, and consider the compatible latent binary allocations of the five cells containing $C$. Initialize the latent $2\times2$ table as
\begin{equation}
M=\begin{pmatrix}a&b\\c&d\end{pmatrix}
=\begin{pmatrix}n_{00}&n_{01}\\n_{10}&n_{11}\end{pmatrix},
\end{equation}
with human labels as rows and judge labels as columns. Then enumerate
\begin{align}
x_{0C}&=0,\ldots,n_{0C}, & x_{1C}&=0,\ldots,n_{1C},\\
x_{C0}&=0,\ldots,n_{C0}, & x_{C1}&=0,\ldots,n_{C1},
\end{align}
and every nonnegative quadruple $(z_{00},z_{01},z_{10},z_{11})$ satisfying $\sum_{h,j}z_{hj}=n_{CC}$. Here $x_{0C}$ and $x_{1C}$ count allocations to latent judge label 1, while $x_{C0}$ and $x_{C1}$ count allocations to latent human label 1. For each allocation, form
\begin{align}
a' &= a+(n_{0C}-x_{0C})+(n_{C0}-x_{C0})+z_{00},\\
b' &= b+x_{0C}+(n_{C1}-x_{C1})+z_{01},\\
c' &= c+(n_{1C}-x_{1C})+x_{C0}+z_{10},\\
d' &= d+x_{1C}+x_{C1}+z_{11}.
\end{align}
This enumeration has
\begin{equation}
(n_{0C}+1)(n_{1C}+1)(n_{C0}+1)(n_{C1}+1)
\binom{n_{CC}+3}{3}
\end{equation}
candidate allocations before duplicate $2\times2$ tables are collapsed.

For each resulting table, evaluate the desired metric when its denominator is nonzero: with $(a',b',c',d')$ in the roles of $(TN,FP,FN,TP)$, the positive-class $F_1$, $\kappa$, and $\phi$ follow from Eq.~\eqref{eq:f1}, the count form of Sec.~\ref{subsec:kappaphi}, and Eq.~\eqref{eq:phi}.
The minimum and maximum over all defined candidate values are the sharp empirical bounds. If no compatible allocation defines a metric, its bound is undefined. The same procedure applies to accuracy or any statistic computed from a binary confusion table.

Like Eq.~\eqref{eq:accuracy-bounds}, these ranges are identification bounds, now conditional on the observed $3\times3$ counts and applied cell by cell rather than to a single bounded outcome. Sampling intervals require an additional procedure, such as resampling items and recomputing both allocation extrema on every bootstrap sample. Reporting a bootstrap interval without the conditional identification range, or reporting the identification range as a confidence interval, conflates the two sources of uncertainty.

For the $2{,}007$-decision rubric counts of Table~\ref{tab:abstention-example}, in which only the human abstains ($n_{0C}=n_{1C}=n_{CC}=0$), joint coverage $\gamma=0.848$ and covered accuracy $0.647$ give the sharp accuracy range $[0.548, 0.701]$; the recoded accuracy $0.580$ lies inside this range while estimating a different quantity.

If the abstention indicators are conditionally independent of the latent verdicts given observed features $X$, written $(m,\hat m)\mathrel{\perp\!\!\!\perp}(y^*,\hat y^*)\mid X$, and every case is covered with positive probability, $\Pr[m=0\wedge\hat m=0\mid X]>0$ almost surely, then reweighting covered cases by $1/\Pr[m=0\wedge\hat m=0\mid X]$ (standardization or inverse-probability weighting) recovers the full latent table \citep{littlerubin2019}.

\subsection{Example: Handling Rules and Bounds}
\label{subsec:case-cascade}

As a case study in how the handling rule alone moves the reported number, we now complete the reconstruction behind the three accuracies that opened this paper. \citet{jung2025trust} build a judge cascade that answers with a cheap model when it is confident and escalates or abstains otherwise, trading coverage for accuracy. Their public release includes human preferences and judge probabilities for $4{,}718$ test pairs, from which we rebuild the full verdict table (Appendix~\ref{app:public-abstention}); the cascade returns a verdict on $2{,}886$ of them, for joint coverage $\gamma=0.612$. The $0.874$ is covered accuracy, agreement restricted to those pairs; the $0.73$ is the recoded accuracy, $0.730$ or $0.727$ depending on which pairwise option absorbs the abstentions; the $0.534$ is three-class accuracy.\footnote{The other metrics move in step (Appendix~\ref{app:public-abstention}); covered $\kappa=\phi$ because the two marginal rates are nearly equal (Eq.~\eqref{eq:ratio}).} The quantity still missing is the bound: with $\gamma=0.612$ and $A_{\mathrm{cov}}=0.874$, Eq.~\eqref{eq:accuracy-bounds} confines full-set forced-choice accuracy to $[0.534, 0.923]$; the lower endpoint equals the three-class accuracy because, with only the cascade abstaining, every abstention counts as a disagreement in both.

These four numbers describe one calibration/test partition, the one the release ships. Redrawing that partition at the released sizes $1{,}000$ times and rerunning the released calibration and decision rule on each (Appendix~\ref{app:public-abstention}) leaves coverage anywhere between $0.412$ and $0.818$ (median $0.628$) and covered accuracy between $0.829$ and $0.936$ (median $0.883$). The pairing of $0.874$ with $\gamma=0.612$ is therefore a property of one draw, not of the cascade: a different split would yield a different pair, and neither the split nor its effect is recoverable from the published number alone.

\subsection{A Rubric Benchmark with Human Abstention}
\label{subsec:case-rubric}

\begin{figure}[t]
\centering
\includegraphics[width=\columnwidth]{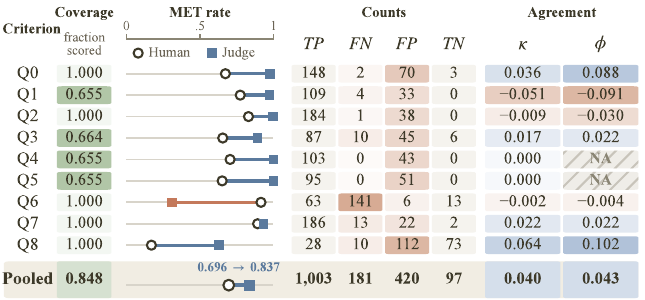}
\caption{Per-criterion agreement between a GPT-3.5-turbo-16k judge and a human annotator on the nine rubric criteria of \citet{hashemi2024llmrubric}, under the reference protocol (most probable option as the verdict, \MET{} $=$ rating $\ge3$, human abstentions excluded). Coverage is the fraction of the $223$ dialogues the human scored for that criterion. Blue dumbbell segments mark a higher judge \MET{} rate, orange a lower one; cell shading scales with magnitude, zero-centered for $\kappa$ and $\phi$. Hatched cells mark the undefined $\phi$ on Q4 and Q5, where the judge never returns \UNMET{}, reported as NA rather than $0$. Appendix~\ref{app:rubric-table} tabulates the exact values.}
\label{fig:rubric}
\end{figure}

The cascade is pairwise, and the side that abstains is the judge. We turn now to per-criterion rubric judgments in which the side that abstains is the human. \citet{hashemi2024llmrubric} release, for $223$ human--AI dialogues, nine rubric criteria rated $1$--$4$ by a human annotator and independently by a GPT-3.5-turbo-16k judge that emits a probability over the same four options: $2{,}007$ paired per-criterion decisions. A human rating of $0$ records that the criterion does not apply, so the human abstains on $306$ decisions ($15.2\%$), while the judge is forced-choice and never abstains. Reading a rating of $3$ or $4$ as \MET{} and the judge's most probable option as its verdict, we obtain the per-criterion tables of Fig.~\ref{fig:rubric}.

The per-criterion tables show what a pooled number hides. The judge's \MET{} rate exceeds the human's on eight of the nine criteria and reaches $1.000$ on Q4 and Q5, where it never returns \UNMET{}; $\phi$ is undefined on those two, and entering $0$ for it rather than NA would quietly bias any average over criteria. Q6 runs the other way, a human \MET{} rate of $0.915$ against the judge's $0.309$, and it is there that the marginal-mismatch factor of Eq.~\eqref{eq:ratio} does its work. Coverage is itself a per-criterion quantity, ranging from $0.655$ to $1.000$ as the human declines to score some criteria far more often than others. Across the nine criteria $\kappa$ runs from $-0.051$ to $0.064$: this judge and this annotator agree at about chance everywhere, which is the finding.

The protocol is a choice at four points: verdict extraction (the judge's most probable option, or the sampled option shipped with the release) and the \MET{} threshold ($\ge2$, $\ge3$, $\ge4$), both within $s$; the abstention rule $h$ (exclusion or recoding); and the aggregation level $a$ (micro or macro). Crossing them gives $24$ protocols on the same $2{,}007$ verdicts. Reported accuracy across them runs from $0.551$ to $0.899$, a spread of $34.8$ points. Extraction alone changes $703$ of the $2{,}007$ verdicts ($35\%$), and the release documents no choice between the two options. Chance-corrected agreement moves less in absolute terms but changes sign, $-0.067$ to $0.068$; because this judge's agreement is near zero under every protocol, that interval is better read as a near-null effect reported with either sign than as a disagreement about the judge's quality. The accuracy spread is the substantive quantity, and it is of the same order as the $34$ points separating the handling rules on the cascade's released split (Sec.~\ref{subsec:case-cascade}), on entirely different data.
\begin{takeaway}
On real rubric data, the four protocol choices move reported accuracy by $34.8$ points and move $\kappa$ across zero, without touching a single verdict. Report the protocol before the number, and report per-criterion tables: pooled agreement here conceals two criteria on which the judge never returns \UNMET{} and $\phi$ does not exist.
\end{takeaway}

\section{Aggregating Criteria, Items, and Judges}
\label{sec:aggregation}

\begin{figure*}[!t]
\centering
\newcommand{\checklistfont}{\small}
\renewcommand{\arraystretch}{1.0}
\scalebox{0.95}{%
\begin{tikzpicture}[
  band/.style={inner sep=4pt, font=\checklistfont},
  chip/.style={fill=#1, rotate=90, rounded corners=3pt, inner sep=3pt, align=center, font=\sffamily\bfseries\fontsize{7}{8}\selectfont, execute at begin node=\hyphenpenalty10000\exhyphenpenalty10000\relax},
  colhead/.style={font=\sffamily\bfseries\fontsize{7}{8}\selectfont, inner sep=2pt, anchor=south west},
]
\newcommand{\bandcards}[2]{%
  \begin{scope}[on background layer]
    \fill[#2, rounded corners=4pt] (#1.north west) rectangle ($(#1.south west)+(7.15cm,0)$);
    \fill[#2, rounded corners=4pt] ($(#1.north west)+(7.36cm,0)$) rectangle (#1.south east);
  \end{scope}}
\node[band] (b1) {\begin{tabular}{@{}>{\raggedright\arraybackslash}p{6.9cm} p{9.1cm}@{}}
1. The judgment scale, before the metric & Each scale admits different statistics; Fact~\ref{obs:equiv} requires binary verdicts. \\
2. At most one of Pearson, Spearman, Kendall's $\tau_b$, $\phi$, and MCC on binary data & They are identical there (Fact~\ref{obs:equiv}); $\phi$/MCC is the most direct. \\
3. The metric matched to the target, with both \MET{} rates & Eq.~\eqref{eq:ratio} fixes $\kappa$'s normalization; other metrics are distinct evidence.
\end{tabular}};
\bandcards{b1}{catscale}
\node[band, below=3pt of b1] (b2) {\begin{tabular}{@{}>{\raggedright\arraybackslash}p{6.9cm} p{9.1cm}@{}}
4. The $2 \times 2$ confusion matrix and $N$ & Summary metrics are invariant under $FP{\leftrightarrow}FN$; the table is not. \\
5. The handling of degenerate criteria & Marked NA, never zero; a zero biases any average over criteria.
\end{tabular}};
\bandcards{b2}{cattables}
\node[band, below=3pt of b2] (b3) {\begin{tabular}{@{}>{\raggedright\arraybackslash}p{6.9cm} p{9.1cm}@{}}
6. The handling rule for \CA{}, ties, and invalid outputs & Excluding, recoding, or retaining \CA{} changes the estimand (Sec.~\ref{subsec:modes}). \\
7. Abstention, tie, invalid, and coverage rates & Equal agreement at unequal coverage is a different operating point. \\
8. Under exclusion, covered performance with joint coverage & The interval of Eq.~\eqref{eq:accuracy-bounds} is identification, not sampling, uncertainty. \\
9. Disagreement costs for multi-class reporting & No \CA{} weight is universal; weights follow from application losses.
\end{tabular}};
\bandcards{b3}{catabst}
\node[band, below=3pt of b3] (b4) {\begin{tabular}{@{}>{\raggedright\arraybackslash}p{6.9cm} p{9.1cm}@{}}
10. The aggregation level and resampling unit & Micro, macro, and item-level scores target different quantities (Sec.~\ref{sec:aggregation}). \\
11. One primary ensemble statistic with supporting tables & The supporting tables carry what a single pooled statistic conceals.
\end{tabular}};
\bandcards{b4}{cataggr}
\node[chip=catscale,  text width=0.95cm] at ($(b1.west)+(-0.51cm,0)$) {Scale \&\\ metric};
\node[chip=cattables, text width=1.18cm]  at ($(b2.west)+(-0.51cm,0)$) {Tables \&\\ marginals};
\node[chip=catabst,   text width=1.45cm]  at ($(b3.west)+(-0.51cm,0)$) {Abstention\\ \& coverage};
\node[chip=cataggr,   text width=1.5cm]  at ($(b4.west)+(-0.51cm,0)$) {Aggregation\\ \& ensembles};
\node[colhead] at ($(b1.north west)+(0.14cm,0.06cm)$) {What to Report};
\node[colhead] at ($(b1.north west)+(7.5cm,0.06cm)$) {Why};
\end{tikzpicture}}%
\caption{Reporting checklist for reconstructible LLM-judge agreement claims; Sec.~\ref{sec:checklist} discusses each item.}
\label{fig:checklist}
\end{figure*}

Because a rubric benchmark produces several criterion verdicts per item, validating a judge requires deciding what one observation is. \emph{Micro-averaging} pools every item--criterion verdict before computing the metric. \emph{Macro-averaging} computes a separate score for each criterion and averages those scores with declared weights. \emph{Item-level aggregation} first combines an item's criterion verdicts into one score, using the rubric's weights and decision rule, and then measures agreement on items.

The rubric benchmark of Sec.~\ref{subsec:case-rubric} shows the effect. Pooling its $1{,}701$ covered decisions gives micro-$\kappa=0.040$; averaging the nine per-criterion values of Fig.~\ref{fig:rubric} gives macro-$\kappa=0.009$, smaller by a factor of more than four. The two summaries weight different things: pooling weights decisions, so criteria the human scored most often dominate, while the macro average weights criteria equally and therefore carries the five criteria whose $\kappa$ is zero or negative. Neither summary reveals the range those nine values span, $-0.051$ to $0.064$, and neither records that $\phi$ does not exist on two of them. Both summaries happen to support the same conclusion about this judge, because both are near zero; the choice bites when two judges are ranked, where the level rather than the sign decides. Item-level $\kappa$ additionally requires each item's full set of criterion verdicts and a reduction rule; the handling of degenerate single-label criteria, the included/total denominator, and the resampling unit all belong in the report (Fig.~\ref{fig:checklist}).
\begin{takeaway}
Report the aggregation level with the number, and never compare judges across levels. A gap between micro and macro estimates signals that agreement varies across criteria; inspect the per-criterion tables before trusting either average.
\end{takeaway}

\subsection{Agreement Within Judge Ensembles}
\label{subsec:ensembles}

Evaluation harnesses often ensemble several judges to reduce prompt sensitivity and model-specific effects \citep{wang2023fairevaluators,verga2024replacing,bavaresco2025llms,rao2026autorubric}. Here we measure agreement \emph{among the member judges}; validating an ensemble's final vote against a human reference is the problem of the preceding sections. Inter-judge disagreement may reflect rubric ambiguity, judge-specific bias, or sampling instability, and a single number alone cannot separate these causes. For $R\ge2$ judges rating each case, common choices include Fleiss' $\kappa$ \citep{fleiss1971measuring} and Krippendorff's $\alpha$ \citep{krippendorff2004content}.

Assume a complete $N\times R$ matrix of binary verdicts with nominal disagreement. Let $\hat{y}_{ij} \in \{0, 1\}$ denote judge $j$'s verdict on item $i$, and let $r_i = \sum_{j=1}^R \hat{y}_{ij}$. With global rates $p_1 = (NR)^{-1} \sum_i r_i$ and $p_0 = 1 - p_1$, Fleiss' $\kappa$ has observed pairwise agreement
\begin{equation}
P_o = \frac{1}{N R(R-1)} \sum_{i=1}^N \bigl[ r_i(r_i - 1) + (R - r_i)(R - r_i - 1)\bigr]
\end{equation}
and expected agreement $P_e = p_0^2 + p_1^2$, giving $\kappa_F = (P_o - P_e)/(1 - P_e)$ \citep{fleiss1971measuring}; $\kappa_F$ is the multi-rater generalization of Scott's $\pi$ \citep{scott1955reliability}. Krippendorff's $\alpha$ on the same data uses observed disagreement $D_o = 1 - P_o$ and expected disagreement $D_e = (1 - P_e) \cdot NR/(NR - 1)$, where the factor $NR/(NR-1)$ is a finite-sample correction \citep{krippendorff2004content}. The two statistics determine each other.

\begin{fact}[\citealp{artstein2008intercoder}]\label{obs:ensemble}
On complete data with $R$ judges and $N$ items under nominal disagreement,
\begin{equation}
\label{eq:obs2}
\alpha = \kappa_F + \frac{1 - \kappa_F}{NR},
\end{equation}
so $\alpha \ge \kappa_F$, with equality iff $\kappa_F = 1$, and $\alpha \to \kappa_F$ as $N \to \infty$.
\end{fact}

That $D_o$ is the complement of Fleiss' observed agreement, and that the expected disagreements differ only by this factor, are established by \citet{artstein2008intercoder}. Substituting these definitions into $\alpha = 1 - D_o/D_e$ and using $1-P_o=(1-\kappa_F)(1-P_e)$, which restates the definition of $\kappa_F$,
\begin{align*}
\alpha &= 1-\frac{NR-1}{NR}\cdot\frac{1-P_o}{1-P_e}
= 1-\frac{NR-1}{NR}\,(1-\kappa_F)\\
&= \kappa_F+\frac{1-\kappa_F}{NR},
\end{align*}
which is the identity of Eq.~\eqref{eq:obs2}. Simulation studies report the same convergence without the closed form \citep{zapf2016measuring}.

Fact~\ref{obs:ensemble} only restates the documented $\alpha$--$\kappa_F$ relationship in closed form; it isolates the finite-sample difference between $\alpha$ and $\kappa_F$: for typical values ($\kappa_F=0.6$, $N=100$, $R=3$) the difference is about $1.3\times10^{-3}$, below the two decimal places at which agreement coefficients are typically reported, but a report should still name which convention it uses. On binary panels Fact~\ref{obs:equiv} applies to each judge pair, and the average pairwise $\phi$ coincides with $\kappa_F$ exactly when all judges share the same marginal \MET{} rate \citep{conger1980integration}; without matched marginals the statistics measure different association structure and should be accompanied by judge-specific marginals or confusion matrices. With three or more categories the identity of Eq.~\eqref{eq:obs2} extends but the binary equivalences do not.

Average pairwise Pearson requires a numeric encoding of $\{$\MET, \UNMET, \CA$\}$; the result depends on the encoding (e.g., $\{1, 0, 1/2\}$ versus deleting \CA{} pairs), and no encoding is canonical. Fleiss' $\kappa$ on three nominal categories weights \MET\textendash{}\UNMET{} and \MET\textendash{}\CA{} confusions equally, while Krippendorff's $\alpha$ with a quadratic-ordinal $\delta$-function gives \MET\textendash{}\UNMET{} greater weight than \MET\textendash{}\CA, which encodes the assumption that \CA{} is intermediate; two defensible cost choices can therefore reverse a ranking of judge ensembles purely through their disagreement functions.

The three-class setup determines no application-independent statistic until disagreement costs are chosen: for values $s=(0,1/2,1)$ assigned to $(\UNMET,\CA,\MET)$, nominal, linear, and quadratic costs are $D_{ab}=\mathbf{1}[a\ne b]$, $|s_a-s_b|$, and $(s_a-s_b)^2$. A complete report states the cost matrix $D$ before any three-class agreement statistic, justifies it from application losses, and includes sensitivity to plausible alternatives and the per-judge $3\times3$ tables (checklist item 9, Fig.~\ref{fig:checklist}).
\begin{takeaway}
Agreement among ensemble members can be high while every member shares the same bias, so it never substitutes for validation against a human reference. When reporting it, name the convention: on the same panel, Krippendorff's $\alpha$ exceeds Fleiss' $\kappa$ by a finite-sample term that vanishes only as the panel grows.
\end{takeaway}

\section{A Reporting Checklist}
\label{sec:checklist}

Figure~\ref{fig:checklist} collects the practices that let a reader reconstruct an agreement claim: its target, its evidence, its scope, and its unit of inference. Because judges compared under different protocols can differ in the reported number while agreeing in behavior, protocol equality precedes any claim that one judge is better. The remainder of this section states each checklist item's full rationale.

\paragraph{Items 1--3: scale and metric.}
Item 1 places the judgment scale before the metric because the scale determines which statistics are defined at all (Sec.~\ref{subsec:scale}): the identity among the five binary agreement metrics is exact only for binary verdicts, and ordinal or graded outputs call for different statistics. Item 2 follows from that identity: on non-degenerate binary vectors, Pearson's $r$, Spearman's $\rho$, Kendall's $\tau_b$, $\phi$, and MCC are one statistic, so a report should give at most one of them, and $\phi$/MCC states binary association most directly. Item 3 matches the remaining metric to the evaluation target and requires both \MET{} rates because the marginal-mismatch factor of Eq.~\eqref{eq:ratio} fixes $\kappa$'s normalization relative to $\phi$/MCC; class-conditional metrics, calibration (when the judge emits probabilities), and the confusion matrix answer different questions and are distinct evidence rather than substitutes.

\paragraph{Items 4--5: tables and marginals.}
Item 4 asks for the $2\times2$ confusion matrix and $N$ because the summary metrics cannot separate a judge that misses true positives from one that over-predicts (Sec.~\ref{sec:protocol}); only the table can, and $N$ fixes the effective sample size behind any interval. For LLM-judge validation, judge strictness or leniency is often the most actionable diagnosis, and the table is what reveals it. Item 5 covers degenerate criteria, on which the judge or the human uses a single label: undefined metrics should be marked NA and never entered as zero (Sec.~\ref{subsec:case-rubric} shows the induced bias on real rubric data); the report should give the constant-label counts and state the macro denominator as included criteria over total criteria, so that comparisons stay on a common eligible set.

\paragraph{Items 6--9: abstention and coverage.}
Item 6 states the handling rule for \CA{}, ties, and invalid outputs because excluding, recoding, or retaining \CA{} changes the estimand, and because ties and invalid outputs, unlike \CA{}, lack a canonical recode-as-negative mapping: a tie is not a loss for either side, and an invalid generation is not an \UNMET{} verdict, so any mapping is a modeling choice that belongs in the report. Item 7 separates the rates: human and judge abstention rates individually, tie and invalid rates, and, when both sides may abstain, the joint coverage, because equal covered agreement at different coverage describes different operating points. Tracked over time, the abstention rate is also a leading indicator: abstention drift is often the first observable symptom of prompt instability or distribution shift. Item 8 pairs covered performance with joint coverage under exclusion because full-set accuracy is identified only to an interval whose width is the uncovered fraction; that interval is identification uncertainty, distinct from sampling error, and reweighting recovers the full table only when observed features explain abstention (Sec.~\ref{subsec:finite-allocation}). Where possible, examine whether abstention correlates with item difficulty, criterion type, prompt, or judge model before treating the covered cases as representative. Item 9 requires disagreement costs for any three-class report because no \CA{} weight is universal (Sec.~\ref{subsec:ensembles}): nominal costs apply when no ordering among the three categories is justified, and otherwise the weights follow from application losses.

\paragraph{Items 10--11: aggregation and ensembles.}
Item 10 names the aggregation level and the resampling unit because micro, macro, and item-level scores target different quantities; the effective $N$ differs across levels, and a cluster bootstrap must resample the unit that carries the dependence, typically whole items. The correspondence runs level by level: a criterion-level cluster bootstrap matches macro-averaging, an item-level cluster bootstrap matches item-level aggregation, and a flat decision-level bootstrap matches only micro-averaging, and then only under the usually false assumption that decisions are independent. Item 11 limits an ensemble report to one primary statistic with supporting tables: with human labels, per-judge confusion matrices; without them, judge-specific marginals and pairwise tables, together with the rules for missing verdicts and judge weighting, because the supporting tables carry the marginal and pairwise structure that a single pooled statistic conceals.

\paragraph{The human reference.}
The checklist validates a judge against a human reference and treats that reference as given; the reference itself can be strengthened through multiple-rater validation and construct-validity analysis \citep{plank2022problem,casabianca2025validity}.

\section{Conclusion}
\label{sec:conclusions}

An LLM-judge agreement number estimates whatever quantity the measurement protocol defines, so examining the protocol is part of evaluating the judge. Treating the scale, subset, handling, and aggregation choices as part of the evaluation target brings four established results to bear: five binary agreement metrics collapse to one statistic, $\kappa$ differs from $\phi$ only through a marginal-mismatch factor \citep{sahu2024linear} with matching asymptotic variance at equal \MET{} rates, exclusion identifies full-set accuracy only up to an interval as wide as the uncovered fraction \citep{manski1989anatomy}, and on complete binary judge panels Krippendorff's $\alpha$ and Fleiss' $\kappa$ determine each other exactly \citep{artstein2008intercoder}. In the cascade of Sec.~\ref{subsec:case-cascade}, that interval spans $0.534$ to $0.923$, and each of the three headline accuracies estimates a different quantity; on the rubric benchmark of Sec.~\ref{subsec:case-rubric}, the four protocol choices alone move reported accuracy by $34.8$ points on one judge's verdicts. The operational sequence follows the analysis: scale, then data subset, then handling rule, then aggregation rule, and only then the statistic. Reporting the contingency tables, marginals, coverage, and effective sample size makes the resulting claim reconstructible.

\medskip
\noindent\textbf{Limitations and Future Work.}
The exact identities require binary verdicts with both labels present (for the other scales, see the taxonomy of Table~\ref{tab:scale}), the equal-variance result is asymptotic and specific to matched \MET{} rates, and hierarchical models could further separate variation due to items, criteria, and judges. Both abstention case studies are unilateral: the cascade abstains on the judge side and the rubric annotator on the human side, and we did not find a released evaluation in which both sides may abstain, which is the case the analysis of Sec.~\ref{sec:abstention} covers in full. The rubric judge is one model on one domain, and its agreement with the annotator is near zero under every protocol, so its protocol spread shows what the choices do to a reported number rather than how a strong judge would fare. The analysis also treats the observed verdict counts as fixed: the identities hold exactly on any dataset on which they are computed, while the inferential statements (the variance equality of Sec.~\ref{subsec:variance} and any bootstrap interval) additionally require a sampling model, which is natural when generalizing to a wider item population or accounting for judge stochasticity at nonzero temperature, and vacuous for a deterministic judge scored once on a fixed evaluation set. Finally, we do not characterize the finite-sample behavior of weighted-$\kappa$ ensembles under arbitrary disagreement-cost matrices, nor the behavior of judge ensembles whose members are not exchangeable, for example one anchored on a frontier model and the rest on small open-weight models.

\medskip
\noindent\textbf{Ethical Statement.}
This work analyzes the relationships among agreement metrics and reports literature coding for 24 papers; it introduces no new benchmark, model, participant data, or human-subjects experiment. We see no major ethical considerations arising from the work itself. The reporting practices in this paper are intended to make comparisons between automated and human judgments transparent and reconstructible.

\medskip
\noindent\textbf{Code Availability.}
Python scripts that reproduce every computed number in this paper, together with the literature-scan coding, accompany the arXiv submission as ancillary files.

\medskip
\noindent\textbf{Acknowledgments.}
This research was developed with funding from the Defense Advanced Research Projects Agency's (DARPA) SciFy program (Agreement No.~HR00112520300). The views expressed are those of the authors and do not reflect the official policy or position of the Department of Defense or the U.S. Government.

\medskip
\noindent\textbf{Generative AI Use Disclosure.}
Gemini-Pro 3.1 and Claude Sonnet 4.6 assisted with proofreading, figure creation, and rewriting for an L2-English author; the corresponding author reviewed all outputs.

\bibliography{paper}

\appendix
\renewcommand{\thetable}{A\arabic{table}}
\setcounter{table}{0}
\renewcommand{\theequation}{A\arabic{equation}}
\setcounter{equation}{0}
\renewcommand{\thefigure}{A\arabic{figure}}
\setcounter{figure}{0}

\section{Public Abstention Reconstruction}
\label{app:public-abstention}

We reconstruct a unilateral abstention example from the public release accompanying \citet{jung2025trust}. The release contains human pairwise preferences and pre-generated probabilities from three LLM judges for 500 calibration pairs and 4,718 test pairs. We used repository commit \texttt{1b969f2}, the released cascade order (Mistral-7B-Instruct, GPT-3.5-Turbo, then GPT-4-Turbo), and the authors' example settings $\alpha=0.15$ and $\delta=0.1$. Applying the released calibration and decision rule to the test split gives Table~\ref{tab:public-abstention-counts}.

\begin{table}[tb]
\centering
\small
\begin{tabular}{@{}lrrr@{}}
\toprule
& \multicolumn{3}{c}{\textbf{Cascade verdict}} \\
\cmidrule(l){2-4}
\textbf{Human option} & \textbf{0} & \textbf{1} & \textbf{Abstain} \\
\midrule
0 & 1,251 & 184 & 921 \\
1 & 181 & 1,270 & 911 \\
\bottomrule
\end{tabular}
\caption{Public cascade reconstruction ($N=4{,}718$). The human reference is binary; only the cascade abstains.}
\label{tab:public-abstention-counts}
\end{table}

The cascade covers 2,886 pairs (coverage $=0.6117$). On that covered subset, accuracy is $0.8735$, Cohen's $\kappa=0.7470$, and $\phi=0.7470$. Table~\ref{tab:public-abstention-strategies} shows how the same counts support different reported values under different handling rules.

\begin{table}[tb]
\centering
\small
\begin{tabular}{@{}lrrr@{}}
\toprule
\textbf{Handling rule} & \textbf{Accuracy} & $\boldsymbol{\kappa}$ & $\boldsymbol{\phi}$ \\
\midrule
Exclude abstentions & 0.8735 & 0.7470 & 0.7470 \\
Recode abstention as 0 & 0.7295 & 0.4594 & 0.4977 \\
Recode abstention as 1 & 0.7274 & 0.4546 & 0.4941 \\
Retain three classes & 0.5343 & 0.3292 & -- \\
\bottomrule
\end{tabular}
\caption{Agreement from the same public counts under the four handling rules. Exclusion uses $N=2{,}886$; the other rows use $N=4{,}718$.}
\label{tab:public-abstention-strategies}
\end{table}

In the exclusion model of Sec.~\ref{subsec:modes}, every abstained pair carries a latent forced-choice verdict pair $(y^*,\hat y^*)$ that abstention masks; under that model, full-population forced-choice accuracy is sharply bounded by $[0.5343,0.9226]$: the lower endpoint treats all 1,832 abstentions as incorrect, and the upper endpoint treats all as correct. This is a public-data special case, not a bilateral \CA{} example. It is also pairwise: options 0 and 1 denote response positions rather than semantic negative and positive classes. Neither recode direction is therefore canonical, which is why both appear in Table~\ref{tab:public-abstention-strategies}.

\paragraph{Dependence on the calibration/test split.}
The counts above come from the single partition the release ships. Because the authors report figures averaged over $1{,}000$ random splits, we pooled the released calibration and test items ($500+4{,}718=5{,}218$ pairs), redrew $1{,}000$ partitions at the same sizes, and reran the released calibration and decision rule on each. Applied to the release's own partition, this procedure returns the values of Table~\ref{tab:public-abstention-strategies} exactly, which anchors the sweep. Table~\ref{tab:public-abstention-splits} reports the resulting distributions.

\begin{table}[t]
\centering
\small
\setlength{\tabcolsep}{4pt}
\begin{tabular}{@{}lrrr@{}}
\toprule
Quantity & Median & IQR & Min\textendash{}Max \\
\midrule
Coverage & 0.6283 & 0.1026 & 0.4118\textendash{}0.8179 \\
Covered accuracy & 0.8827 & 0.0258 & 0.8287\textendash{}0.9357 \\
Covered $\kappa$ & 0.7655 & 0.0516 & 0.6574\textendash{}0.8713 \\
Recode as 0, accuracy & 0.7344 & 0.0309 & 0.6776\textendash{}0.7891 \\
Recode as 1, accuracy & 0.7378 & 0.0300 & 0.6789\textendash{}0.7808 \\
Three-class accuracy & 0.5521 & 0.0800 & 0.3853\textendash{}0.6823 \\
\midrule
Within-split range over rules & 0.3269 & 0.0967 & 0.1519\textendash{}0.5503 \\
\bottomrule
\end{tabular}
\caption{Distribution over $1{,}000$ redrawn calibration/test splits of the released data. The last row is the range the four handling rules span within a single split.}
\label{tab:public-abstention-splits}
\end{table}

Coverage is the least stable quantity: with only $500$ calibration pairs, the calibrated thresholds move enough to place coverage anywhere between $0.41$ and $0.82$. Three-class accuracy inherits that instability, since it is coverage multiplied by covered accuracy. The pairing of a covered accuracy with a coverage, which is how selective judges are usually reported, is therefore a property of the split as much as of the cascade, and neither is recoverable from a published pair alone.

All reconstructions and reanalyses in this paper are deterministic, CPU-only computations in Python with \texttt{numpy} and \texttt{scipy}; the split sweep completes in about a minute on a laptop and every other analysis in seconds. The code provided as ancillary material lists the exact library versions and reproduces every reported value.

\section{Per-Criterion Rubric Benchmark Numbers}
\label{app:rubric-table}

Table~\ref{tab:rubric-full} tabulates the exact values summarized in Fig.~\ref{fig:rubric}: per-criterion coverage, human and judge \MET{} rates, verdict counts, and agreement under the reference protocol.

\begin{table*}[!t]
\centering
\small
\begin{tabular}{@{}lrrrrrrrrr@{}}
\toprule
& & \multicolumn{2}{c}{\MET{} rate} & \multicolumn{4}{c}{Counts} & & \\
\cmidrule(lr){3-4}\cmidrule(lr){5-8}
Criterion & Coverage & Human & Judge & $TP$ & $FN$ & $FP$ & $TN$ & $\kappa$ & $\phi$ \\
\midrule
Q0 & 1.000 & 0.673 & 0.978 & 148 & 2 & 70 & 3 & \phantom{-}0.036 & \phantom{-}0.088 \\
Q1 & 0.655 & 0.774 & 0.973 & 109 & 4 & 33 & 0 & -0.051 & -0.091 \\
Q2 & 1.000 & 0.830 & 0.996 & 184 & 1 & 38 & 0 & -0.009 & -0.030 \\
Q3 & 0.664 & 0.655 & 0.892 & 87 & 10 & 45 & 6 & \phantom{-}0.017 & \phantom{-}0.022 \\
Q4 & 0.655 & 0.705 & 1.000 & 103 & 0 & 43 & 0 & \phantom{-}0.000 & NA \\
Q5 & 0.655 & 0.651 & 1.000 & 95 & 0 & 51 & 0 & \phantom{-}0.000 & NA \\
Q6 & 1.000 & 0.915 & 0.309 & 63 & 141 & 6 & 13 & -0.002 & -0.004 \\
Q7 & 1.000 & 0.892 & 0.933 & 186 & 13 & 22 & 2 & \phantom{-}0.022 & \phantom{-}0.022 \\
Q8 & 1.000 & 0.170 & 0.628 & 28 & 10 & 112 & 73 & \phantom{-}0.064 & \phantom{-}0.102 \\
\midrule
Pooled & 0.848 & 0.696 & 0.837 & 1{,}003 & 181 & 420 & 97 & \phantom{-}0.040 & \phantom{-}0.043 \\
\bottomrule
\end{tabular}
\caption{Numeric form of Fig.~\ref{fig:rubric}: per-criterion coverage, \MET{} rates, verdict counts, $\kappa$, and $\phi$ for the nine rubric criteria of \citet{hashemi2024llmrubric} under the reference protocol (most probable option as the verdict, \MET{} $=$ rating $\ge3$, human abstentions excluded). $\phi$ is undefined on Q4 and Q5, where the judge never returns \UNMET{}, and is reported as NA rather than $0$.}
\label{tab:rubric-full}
\end{table*}

\end{document}